\documentclass{new_tlp}

\title[Modular Constraint Solver Cooperation via Abstract Interpretation]
      {Modular Constraint Solver Cooperation\\ via Abstract Interpretation$^\text{\footnotemark{}}$
      \footnotetext{This work was partially supported by ANR-15-CE25-0002 Coverif from the French \textit{Agence Nationale de la Recherche}. The \textit{Centre de calcul intensif des Pays de la Loire} (CCIPL) provided the infrastructure to perform the benchmarks. The authors thank the anonymous reviewers for their constructive comments to improve the clarity of the paper. We thank Yinghan Ling for the English proofreading.}}
\author[P. Talbot, E. Monfroy and C. Truchet]
       {PIERRE TALBOT\\
       Interdisciplinary Centre for Security, Reliability and Trust (SNT),\\
       University of Luxembourg, Esch-sur-Alzette, Luxembourg
        \email{pierre.talbot@uni.lu}
        \and \'{E}RIC MONFROY\\
        University of Angers, Angers, France
        \email{eric.monfroy@univ-angers.fr}
        \and CHARLOTTE TRUCHET\\
        University of Nantes, Nantes, France
        \email{charlotte.truchet@univ-nantes.fr}}

\jdate{August 2020}
\pubyear{2020}
\pagerange{\pageref{firstpage}--\pageref{lastpage}}
\doi{XXXX}

\usepackage{graphicx}
\usepackage{latexsym}
\usepackage{amsmath,amssymb,amsfonts}
\usepackage{listings}
\usepackage[hidelinks]{hyperref}
\usepackage{algpseudocode}
\usepackage{textcomp}
\usepackage{xcolor}
\usepackage{url}
\usepackage{stmaryrd}
\usepackage{gensymb}
\usepackage{tikz}
\usetikzlibrary{arrows.meta}
\usepackage{subfig}
\usepackage{mathtools}
\usepackage[round]{natbib}
\renewcommand{\cite}[1]{\citep{#1}}

\newcommand{\pf}[2]{[#1 \nrightarrow #2]}

\def\vars{V}

\newcommand{\cd}[1]{#1^{\flat}}
\newcommand{\ad}[1]{#1^{\sharp}}
\newcommand{\cdi}[1]{\cd{\llbracket #1 \rrbracket}}
\newcommand{\adi}[1]{\ad{\llbracket #1 \rrbracket}}

\def\ttrue{\mathit{true}}
\def\ffalse{\mathit{false}}
\def\unk{\mathit{unknown}}

\newcommand\eqdef{\,\triangleq\,}

\definecolor{bleu}{rgb}{0,0.6,0.8}
\definecolor{rouge}{rgb}{1.,0.,0.2}
\definecolor{vert}{rgb}{0,0.4,0.4}

\definecolor{ForestGreen}{rgb}{.132,.545,.132}

\newcounter{thm}
\newtheorem{definition}[thm]{Definition}
\newtheorem{lemma}[thm]{Lemma}
\newtheorem{proposition}[thm]{Proposition}

\begin{document}

\maketitle

\begin{abstract}
Cooperation among constraint solvers is difficult because different solving paradigms have different theoretical foundations.
Recent works have shown that abstract interpretation can provide a unifying theory for various constraint solvers.
In particular, it relies on abstract domains which capture constraint languages as ordered structures.
The key insight of this paper is viewing cooperation schemes as abstract domains combinations.
We propose a modular framework in which solvers and cooperation schemes can be seamlessly added and combined.
This differs from existing approaches such as SMT where the cooperation scheme is usually fixed (\emph{e.g.}, Nelson-Oppen).
We contribute to two new cooperation schemes: (i) \textit{interval propagators completion} that allows abstract domains to exchange bound constraints, and (ii) \textit{delayed product} which exchanges over-approximations of constraints between two abstract domains.
Moreover, the delayed product is based on delayed goal of logic programming, and it shows that abstract domains can also capture control aspects of constraint solving.
Finally, to achieve modularity, we propose the \textit{shared product} to combine abstract domains and cooperation schemes.
Our approach has been fully implemented, and we provide various examples on the flexible job shop scheduling problem.
Under consideration for acceptance in TPLP.
\end{abstract}


\section{Introduction}

% Context
% Constraint programming aims at solving combinatorial problems expressed in a declarative way.
% It is applied to a large series of real-life problems, including vehicle routing problems, scheduling problems, and biological or musical problems~\cite{handbook-cp,truchet-constraint-2011}.
% The user specifies its problem in a mathematical formalism, and a solution is automatically found by a constraint solver.
% Therefore, the user does not need to program a resolution procedure that describes \textit{how} to solve the problem, and can focus on \textit{what} the problem is.

% Specialized context
A constraint solver is often more efficient when it targets at a specific constraint language, such as satisfiability (SAT) solvers with Boolean formulas, or  linear  programming solvers with linear arithmetic constraints.
However, problem specifications often consist of constraints of different types.
A real-life problem can contain two or more constraints such that one is more efficiently treated or it can only be treated in a solver, and the other one in another solver.
In such case, it is necessary to find a solver that supports all the constraints of the problem, but it may not be as efficient as specialized solvers.
Therefore, the cooperation among solvers becomes a central concern in order to achieve better efficiency and to improve expressiveness of the solvers.
Satisfiability modulo theories (SMT) solvers are probably the most well-known cooperation framework as they encapsulate constraint languages in theories that can be combined together by the Nelson-Oppen scheme~\cite{nelson-oppen-1979}.
Lazy clause generation~\cite{Ohrimenko:2009:PVL:1553323.1553342} is a more specialized example that mixes SAT solving and propagation-based constraint solvers, and currently it is the state of the art solver for many scheduling problems. %~\cite{schutt-rcpsp-max-lcg-2013}.
On the other end of the spectrum, the black box approaches study the combination of solvers without modifying them~\cite{bali-monfroy-1998}.
Overall, the combination of two or more solvers often results in a third solver with little consideration about the modularity and the reuse of its components and its cooperation scheme.

% Why is it interesting
We propose a theoretical and practical framework for constraint solving, where it is possible to introduce new solvers and cooperation schemes in a modular way.
In comparison to previous work, our cooperation schemes between solvers are not built in the framework itself, but defined at the same level as solvers.
It enables us to define various cooperation schemes among solvers, which can be run concurrently.

% How do you solve it
Our proposal is based on abstract interpretation~\cite{cousot-abstract-1977}, a framework to perform static analysis of programs.
Abstract domains are an important fragment of abstract interpretation.
They capture constraint languages as ordered structures.
Abstract interpretation has the advantage to cleanly separate between the logical formula (syntax), the abstract domain (semantics representable in a machine), and the concrete domain (mathematical semantics)---we introduce these concepts in Section~\ref{abstract-interpretation-cp}.
This separation and the order theory underlying abstract domains help to prove mathematical properties on the combination of abstract domains.
Moreover, abstract domains can be implemented almost directly as they describe the semantics of the solvers.
% \vspace{-0.2cm}
\paragraph{Contributions} This paper focuses on \textit{domain transformers}, which are functors constructing abstract domains from one or more abstract domains.
We propose two domain transformers capturing two cooperation schemes.
Firstly, the \textit{interval propagators completion} ($\mathsf{IPC}$) which equips any abstract domain with \textit{interval propagators} (Section~\ref{pp}).
An interval propagator is a function implementing an arithmetic constraint (linear or non-linear).
This completion can be applied to products of domains, which results in a cooperation scheme where two domains exchange bound constraints over their shared variables.
Secondly, we propose the \textit{delayed product} ($\mathsf{DP}$) which treats a constraint $c$ in an abstract domain $A_1$ until $c$ becomes treatable in a more efficient abstract domain $A_2$ (Section~\ref{delayed-product}).
This technique is inspired by the delayed goal technique of logic programming.
The delayed product dynamically rewrites a constraint once its variables are instantiated.
In addition, over-approximations of this constraint can be incrementally sent to $A_2$ before the variables of $c$ are fully instantiated.
Finally, we introduce the \textit{shared product} to combine domain transformers sharing abstract domains (Section~\ref{coop-scheme}).
This product enables the hierarchy of abstract domains and transformers to form a directed acyclic graph.
We illustrate these abstract domains over the flexible job shop scheduling problem in Section~\ref{impl-experiments}.
In particular, we reveal that several constraint solvers can be obtained by assembling the presented abstract domains, and that they are competitive with state of the art approaches.
\vspace{-0.2cm}
\paragraph{Related works}

Cooperation schemes between the domain of uninterpreted functions (Herbrand universe of a logic program) and various constraint systems have been widely studied in the context of \textit{constraint logic programming} (CLP).
For instance, \texttt{CLP(BNR)} deals with mixed continuous and discrete domains~\cite{older-programming-clp-bnr-1993}.
$\mathcal{TOY}$ is a functional CLP language that aims at the solvers cooperation among uninterpreted functions, arithmetic constraints over real numbers, and finite domains~\cite{toy-cooperation-2009}.
In particular, $\mathcal{TOY}$ introduces the notion of \textit{bridges}, such as \texttt{X \#==$_\mathit{int,real}$ Y} between two variables such that $X$ is an integer and $Y$ a real.
The transformer $\mathsf{IPC}$ can be seen as implementing generic bridges among its underlying domains.
A downside of CLP approaches is that the addition of a new constraint system or combination often corresponds to the design of a new language.

The SMT paradigm is an important field about theory combination at the logical level.
Theory and abstract domain are two sides of the same coin: theory captures the logical essence of a constraint language, while abstract domain captures its semantics.
In fact, it was shown that Nelson-Oppen combination is a specific reduced product, a technique to combine abstract domains, in abstract interpretation~\cite{cousot-theories-2012}.
Deeper connections have been made by the abstract conflict driven clause learning (ACDCL) framework~\cite{dsilva-abstract-2014} which demonstrates that SMT solvers can be considered as fixed point computation over abstract domains.
ACDCL is mostly a theoretical proposal and it has not been thoroughly investigated in practice.
Overall, cooperation schemes are either built in the theory or left aside in both SMT and ACDCL frameworks.

\section{Abstract interpretation for constraint programming}
\label{abstract-interpretation-cp}

Abstract interpretation is a framework to statically analyze programs by over-approximating the set of values that the variables of the program can take~\cite{cousot-abstract-1977}. % (see also~\cite{mine-tutorial-2017} for an introduction).
In a nutshell, the following diagram presents the fragment of abstract interpretation we are interested in:
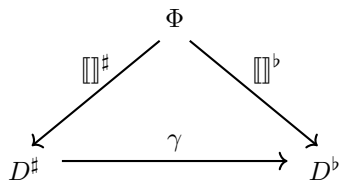
\begin{center}
\begin{tikzpicture}
  \draw [thick,->] (-0.2,1.7) -- node[midway,above] {$\adi{}$} (-1.9,0.3);
  \draw [thick,->] (0.2,1.7) -- node[midway,above] {$\quad\cdi{}$} (1.9,0.3);
  \draw [thick,->] (-1.5,0.1) -- node[midway,above] {$\gamma$} (1.5,0.1);
  % \draw [thick,<-] (-1.5,-0.1) -- node[midway,below] {$\alpha$} (1.5,-0.1);

  \node at (0,2) {$\Phi$};
  \node at (-2,0) {$\ad{D}$};
  \node at (2,0) {$\cd{D}$};
\end{tikzpicture}
\end{center}
This diagram connects a logical formula, a concrete domain and an abstract domain.
The syntax of a program, specifically in our case, of a constraint problem is represented by the set $\Phi$ of any quantifier-free first-order logic formulas.
We interpret a formula $\varphi$ to a concrete or abstract domain respectively with $\cdi{\varphi}$ and $\adi{\varphi}$.
The concrete domain represents the mathematical semantics of this formula, its exact set of solutions which may be infinite and not computer-representable.
The abstract domain corresponds to the machine semantics of this formula that might under- or over-approximate the set of solutions of the concrete domain.
Approximations are particularly insightful on continuous domains, such as real numbers, which have to be approximated using floating point numbers.
An abstract domain is connected to the concrete domain by a concretization function $\gamma: \ad{D} \to \cd{D}$, which is useful to prove properties of the abstract domain\footnote{Abstract interpretation usually relies on an abstraction function $\alpha: \cd{D} \to \ad{D}$.
In our case, the concrete solutions set is fixed and always given by $\cdi{\varphi}$, thus we have $\alpha(\cdi{\varphi}) = \adi{\varphi}$.}.
In the following, as we mainly manipulate abstract domains, we will omit the $\sharp$ symbol on the operators, for instance $\adi{.}$ is written as $\llbracket . \rrbracket$.
This section summarizes previous work~\cite{pelleau-constraint-2013,ICTAI19-Talbot} in which more formal definitions and proofs can be found.

% We first present the concrete domain of constraint satisfaction problems, which is followed by the definition of abstract domains.
% The rest of this section overviews the box and octagon abstract domains and their completion to logic formula.
% We then define a constraint solver as a fixed point function over an abstract domain.

\paragraph{Concrete domain}

A constraint satisfaction problem (CSP) is a tuple $(X, D, C)$ where $X$ is a set of variables, $D = D_1 \times \ldots \times D_n$ the sets of values taken by each variable $x_i \in X$, and $C$ a set of relations over variables, called \textit{constraints}.
A constraint $c \in C$, defined on the variables $x_1,\ldots,x_n$ is satisfied when $c(v_1,\ldots,v_n)$ holds for all $v_i \in D_i$.
The \textit{concrete domain} is the powerset lattice $\cd{D} = \langle \mathcal{P}(D), \supseteq \rangle$ ordered by inclusion.
The concrete interpretation function maps a CSP $(X,D,C)$\footnote{Note that $(X,D,C)$ is just a structured presentation of a logical formula.} to an element in $\cd{D}$ representing its set of solutions:
\begin{displaymath}
\cdi{(X,D,C)} = \{(D'_1,\ldots,D'_n) \;|\; D'_i \subseteq D_i \text{ and all $c \in C$ satisfied}\}
\end{displaymath}
% We keep the terminology ``constraints'' when referring to terms of a logical formula.
% \vspace{-0.8cm}
\paragraph{Abstract domain}

In abstract interpretation, an abstract domain is a partially ordered set equipped with useful operations for programs analysis.
This notion has been adapted to constraint programming, where some operators are reused (\emph{e.g.}, join and interpretation function) and some are new (\emph{e.g.}, state and split) for its application to constraint solving.
In the following, ``abstract domain'' will refer to this modified notion of abstract domain for constraint programming.
The set $K = \{\ttrue,\ffalse,\unk\}$ represents elements of Kleene logic, in which we have $\ffalse \,\land\, \unk = \ffalse$ and $\ttrue \,\land\, \unk = \unk$.

\begin{definition}[Abstract domain]
An abstract domain for constraint programming is a lattice $\langle A, \leq \rangle$ where $A$ is a set of computer-representable elements equipped with the following operations:
\begin{itemize}
  \item $\bot$ is the smallest element and, if it exists, $\top$ the largest.
  \item $\sqcup: A \times A \to A$ is called the \textit{join}, it performs the union of the information contained in two elements.
  % Dually, we have $\sqcap: A \times A \to A$, called the \textit{meet}, and performing the intersection of the information.
  \item $\gamma: A \to \cd{D}$ is a monotonic \textit{concretization} function mapping an abstract element to its set of solutions.
  \item $state: A \to K$ gives the state of an element:
   $\ttrue$ if the element satisfies all the constraints of the abstract domain, $\ffalse$ if at least one constraint is not satisfied, and $\unk$ if satisfiability cannot be established yet.
  \item $\llbracket . \rrbracket : \Phi \to A$ is a partial function transferring a formula to an element of the abstract domain\footnote{Alternatively, this function could be total and every unsupported formula mapped to $\bot$ which is a correct over-approximation.
  However, it prevents us from distinguishing between tautological formulas (since $\llbracket \ttrue \rrbracket = \bot$) and unsupported formulas.
  In the first case, we wish to interpret the formula in $A$, while in the second case we prefer to look for another, more suitable, abstract domain.}.
  This function is not necessarily defined for all formulas since an abstract domain efficiently handles a delimited constraint language.
  \item $\mathit{closure}: A \to A$ is an extensive function ($\forall x, x \leq \mathit{closure}(x)$) which eliminates inconsistent values from the abstract domain.
  % In constraint programming, $closure$ is known as propagation.
  \item $\mathit{split}: A \to \mathcal{P}(A)$ divides an element of an abstract domain into a finite set of sub-elements.
\end{itemize}
\label{abstract-dom}
\end{definition}
\noindent
We refer to the ordering of the lattice $L$ as $\leq_L$ and similarly for any operation defined on $L$, unless no confusion is possible.
An abstract element $a \in A$ under-approximates the concrete solutions set of a formula $\varphi$ if $\gamma(a) \subseteq \cdi{\varphi}$, which implies that all points in $a$ are solutions, but solutions might be missing.
Dually, $a$ over-approximates $\varphi$ if $\gamma(a) \supseteq \cdi{\varphi}$, which implies that all solutions are preserved but there might be non-solution points in $a$.
We can prove properties on abstract domains by verifying these two equations.

We present an algorithm to refine the approximation of an element $\llbracket \varphi \rrbracket \in A$ approximating a formula $\varphi$.
This algorithm is generic over an abstract domain $A$.
\begin{algorithmic}[1]
\Function{$\mathtt{solve}$}{$a \in A$}
\State $a \leftarrow \mathtt{closure}(a)$
\If{$\mathtt{state}(a) = \mathtt{true}$} \Return $\{a\}$
\ElsIf{$\mathtt{state}(a) = \mathtt{false}$} \Return $\{\}$
\Else
  \State $\langle a_1,\ldots, a_n \rangle \leftarrow \mathtt{split}(a)$
  \State \Return $\bigcup_{i=0}^{n} \mathtt{solve}(a_i)$
\EndIf
\EndFunction
\end{algorithmic}
\noindent
This algorithm follows the usual solving pattern in constraint programming which is \textit{propagate and search}.
We infer as much information as possible with $\mathit{closure}$, and then divide the problem into sub-problems with $\mathit{split}$.
We rely on $\mathit{state}$ for the base cases defined when we reach a solution or an inconsistent node.
We obtain the solutions of a constraint set $C$ in an abstract domain $A$ with $\mathit{solve}(\bigsqcup_{c \in C} \llbracket c \rrbracket)$.
It is noteworthy that the abstract domain $a \in A$ can be a composition of several abstract domains through domain transformers (see Section~\ref{domain-transformers}).
The over-approximation property extends to $\mathtt{solve}$ with $\bigcup\{\gamma(a) \;|\; a \in \mathtt{solve}(\llbracket \varphi \rrbracket)\}\supseteq \cdi{\varphi}$, and dually for under-approximation.
A termination condition and proof of this algorithm are given in~\cite{ICTAI19-Talbot}.
We illustrate these definitions in some abstract domains as follows.

\paragraph{Box abstract domain}

We denote by $\mathbb{A}$ the set of integers $\mathbb{Z} \cup \{-\infty,\infty\}$.
An interval is a pair $(l,u) \in \mathbb{A}^2$ of the lower and upper bounds, written as $[l..u]$, defined as $\gamma([l..u]) = \{ x \in \mathbb{Z} \;|\; l \leq x \leq u\}$.
The set of intervals $I = \langle \{[l..u] \;|\; \forall{l,u} \in \mathbb{A}\}, \leq, \bot,\top,\sqcup \rangle$ is a lattice ordered by set inclusion $\leq \eqdef \supseteq$.
It has a bottom element $\bot \eqdef [-\infty,\infty]$, a top element $\top \eqdef \{\}$, and a join $\sqcup \eqdef \cap$ defined by set intersection.
An interval can be used to represent the domain of a single variable.
In order to represent collection of variable's domains, we consider the lattice of partial functions $\pf{V}{I}$ from the set of variable's names $V$ to the lattice of intervals $I$.
Practically, elements of $\pf{V}{I}$ can be thought as arrays of interval domains.
This lattice is studied by~\citet{fernandez-interval-2004} for constraint solving in a more general setting.
The \textit{box abstract domain} $\mathcal{B} = \langle \pf{V}{I}, \leq \rangle$ equips $\pf{V}{I}$ with the operators of Def.~\ref{abstract-dom}.
Boxes capture a small constraint language consisting of the constraints $x \leq b$, $x \geq b$, $x < b$, $x > b$ and $x = b$, where $x \in V$, $b \in \mathbb{A}$.
The role of the interpretation function is then to map each supported constraint to an element of the abstract domain.
The logical conjunction coincides with the join in the lattice.
For instance we have $B = \llbracket x > 2 \land x \leq 4 \land y > 0 \rrbracket = \llbracket x > 2 \rrbracket \sqcup \llbracket x \leq 4 \rrbracket \sqcup \llbracket y > 0 \rrbracket = \{x \mapsto [3..\infty]\} \sqcup \{x \mapsto [-\infty..4]\} \sqcup \{y \mapsto [1..\infty]\} = \{x \mapsto [3..4], y \mapsto [1..\infty]\}$.
The concrete set of elements is obtained by listing all solutions, \emph{e.g.}, $\gamma(B) = \{\{(x,3),(y,1)\},\{(x,4),(y,1)\},\{(x,3),(y,2)\},\ldots\}$.
We observe that this set is infinite, which is why we need an abstract domain approximating this set in a finite way.
In the case of boxes, the interpretation function is both an under- and over-approximation because, for all formulas $\varphi$ such that $\llbracket \varphi \rrbracket$ is defined, we have $\gamma(\llbracket \varphi \rrbracket) \subseteq \cdi{\varphi}$ and $\gamma(\llbracket \varphi \rrbracket) \supseteq \cdi{\varphi}$.
Therefore, once a constraint has been interpreted, we have the best possible approximation, and thus $\mathit{closure}$ is simply the identity function.
It is not always the case, as explained below with octagons.
An element $b \in \mathcal{B}$ is consistent if $\gamma(a) \neq \{\}$, which boils down to the observation that no variable has an empty interval.
This observation can be used to define $\mathit{state}$, which is always equal to either $\ttrue$ or $\ffalse$.
The $\mathit{split}$ operator can only be useful when boxes are used in combination with other abstract domains.
It can be defined by selecting a variable $x = [l..u]$ in $b \in \mathcal{B}$ and dividing this interval into two parts, \emph{e.g.},~$\mathit{split}(b) = \{b \sqcup \llbracket x = l \rrbracket, b \sqcup \llbracket x > l \rrbracket\}$.
It is worth mentioning that many different $\mathit{split}$ operators are possible, which are more or less efficient depending on the problem at hand.

\paragraph{Octagon abstract domain}

The octagon abstract domain~\cite{mine-octagon-2006}, denoted by $\mathsf{O}$, is more expressive than boxes because it can interpret constraints of the form $\pm x \pm y \leq c$ and $\pm x \leq c$ where $x,y$ are variables and $c$ is a constant (either over integers, floating point numbers or rational numbers).
Internally, an octagon is represented by a difference-bound matrix of size $\mathcal{O}(n^2)$ where $n$ is the number of variables.
Its closure operator is the Floyd-Warshall algorithm which runs in $\mathcal{O}(n^3)$ in the general case.
An incremental version in $\mathcal{O}(n^2)$ is available when only one constraint is added.
% It is sometimes computationally cheaper to first interpret all constraints in a bulk and then perform the closure, rather than performing the closure immediately when interpreting\footnote{Consider two successive constraints $x + y \leq 1$ and $x + y \leq 0$ where the second one subsumes the first.
% The first closure would be unnecessary as it is only useful to take into account the second constraint.}.
% This is the reason why we separated the interpretation from the closure operator.
% A property of octagons is that $\mathit{closure}$ always maps an element to a consistent or inconsistent element, i.e.~$\mathit{state}(\mathit{closure}(o))$ is either equal to $\ttrue$ or $\ffalse$.
% Hence, the algorithm $\mathtt{solve}$ only performs a single step.

\paragraph{A domain transformer: logic completion}

The logic completion $\mathsf{L}(A)$ is a domain transformer: it takes an abstract domain $A$ as a parameter and produces a new abstract domain supporting logical connectors over the constraint language of $A$.
For example, the formula $c_1 \eqdef (x = 1 \lor x = 2)$ is neither interpretable in boxes nor in octagons, but it is in $\mathsf{L}(\mathcal{B})$ or $\mathsf{L}(\mathsf{O})$.
In the presence of disjunction, we have $\mathit{state}(\mathit{closure}(\llbracket c_1 \rrbracket)) = \unk$, because there is too few information to infer whether $x = 1$ or $x = 2$.
A choice must be made, and this is where the $\mathit{split}$ operator and the $\mathtt{solve}$ algorithm become necessary.
The problem is decomposed into two subproblems $x = 1$ and $x = 2$ which are solved in turn.
The union of their solutions is the solutions set of the initial problem.

\paragraph{Combination of domains: direct product}

The value of this abstract framework stands out when abstract domains are combined.
For instance, consider the formula $c_2 \eqdef (x > 4 \land x < 7) \Rightarrow y + z \leq 4$.
The abstract domain $\mathsf{L}(\mathsf{O})$ is expressive enough to interpret $c_2$.
However, the constraints on $x$ can be treated more efficiently in boxes than in octagons due to the lower space complexity of boxes.
Therefore, it is advantageous to interpret $x > 4 \land x < 7$ in a box and $y + z \leq 4$ in an octagon.
In order to achieve that, we rely on the direct product $\mathcal{B} \times \mathsf{O}$.
When stacked with the logic completion transformer, it gives us $\mathsf{L}(\mathcal{B} \times \mathsf{O})$.
We define the direct product as follows.
\newpage
\begin{definition}[Direct product]
Let $A_1,\ldots,A_n$ be a collection of $n$ abstract domains.
The direct product is an abstract domain $\langle A_1 \times \ldots \times A_n, \leq \rangle$ where each operator is defined coordinatewise, \emph{e.g.}, $(a_1,\ldots,a_n) \leq (b_1,\ldots,b_n) \Leftrightarrow \bigwedge_{1 \leq i \leq n} a_i \leq_i b_i$, with $(a_1,\ldots,a_n), (b_1,\ldots,b_n) \in A_1 \times \ldots \times A_n$.
% \begin{itemize}
%   % \item $(a_1,\ldots,a_n) \sqcup (b_1,\ldots,b_n) \eqdef (a_1 \sqcup_1 b_1,\ldots,a_n \sqcup_n b_n)$.
%   % \item $\bot \eqdef (\bot_1,\ldots,\bot_n)$ and $\top \eqdef (\top_1,\ldots,\top_n)$.
%   \item $\gamma((a_1,\ldots,a_n)) \eqdef \bigcup_{1 \leq i \leq n} \gamma_i(a_i)$
%   \item $\llbracket c \rrbracket \eqdef (a_1,\ldots,a_n)$ where $a_i = \llbracket c \rrbracket_i$ if defined and otherwise $\bot_i$.
%   % \item $\mathit{closure}((a_1,\ldots,a_n)) \eqdef (\mathit{closure}_1(a_1),\ldots,\mathit{closure}_n(a_n))$.
%   \item $\mathit{state}((a_1,\ldots,a_n)) \eqdef \bigwedge_{1 \leq i \leq n} \mathit{state}_i(a_i)$.
%   \item $\mathit{split}((a_1,\ldots,a_n)) \eqdef \{(a_1,\ldots,b_i,\ldots,a_n)\,|\, b_i \in \mathit{split}_i(a_i), |b_i| > 1\}$
% \end{itemize}
\end{definition}

There is a small issue about the previous formula $c_2$: the constraint $x > 4 \land x < 7$ is interpretable both in boxes and octagons.
In this case, the direct product will interpret this formula in both domains, which is not the behavior we expect.
To solve this problem, we annotate formulas with an integer, denoted as \texttt{$\varphi$:i}, meaning that $\varphi$ should be interpreted in the $i^\text{th}$ component of the product.
Formally, we have $\llbracket \varphi\mathtt{:i} \rrbracket \eqdef (\bot_1,\ldots,\llbracket \varphi \rrbracket_i, \ldots, \bot_n)$.
The formula can be duplicated as many times as needed to be interpreted in more than one domain.
If we annotate $c_2$ with $(x > 4 \land x < 7)\mathtt{:1} \Rightarrow (y + z \leq 4)\mathtt{:2}$, the first constraint will be interpreted in boxes and the second one in octagons---note that the logic completion forwards the interpretation of annotated sub-formulas to the underlying domain, here the product.

The cooperation happening between the box and octagon domains in $\mathsf{L}(\mathcal{B} \times \mathsf{O})$ is fully logical.
This form of cooperation allows us to address some complex problems, as shown in~\cite{ICTAI19-Talbot}.
However, as soon as two constraints belonging to different abstract domains share variables, the variable's domains (\emph{e.g.}, intervals) are not shared among the domains.
Indeed, the operators of the direct product defined coordinatewise, each $\mathit{closure}_i$ is independently applied to each component of the product, but the new information obtained is never exchanged.
In the next section, we propose two domain transformers that exchange information between domains in two different ways.
As a cross-product, we show that domain transformers also capture operational aspects (such as delayed goals), that are more difficult to express in a fully logical setting.

\section{Domain transformers for cooperation schemes}
\label{domain-transformers}

% To illustrate that the shared product meets our modularity and compositionality requirements, we present two cooperation schemes that are defined as domain transformers.
% As a cross-product, we show that domain transformers also capture operational aspects (such as delayed goals), that are more difficult to express in a purely logical setting.
% \vspace{-0.3cm}
\subsection{Interval propagators completion}
\label{pp}

Consider the constraint $c_3 \eqdef x > 1 \land x + y + z \leq 5 \land y - z \leq 3$.
The constraint $x > 1$ can be interpreted in boxes and $y - z \leq 3$ in octagons, but $x + y + z \leq 5$ is too general to be interpreted in any abstract domain we introduced until now.
Moreover, the last constraint shares a variable with the other two.
The interval propagators completion is a domain transformer, denoted as $\mathsf{IPC}(A)$, which solves both problems at once.
$\mathsf{IPC}(A)$ extends the constraint language of any abstract domain $A$ to arbitrary arithmetic constraints.
The constraint $(x > 1)\mathtt{:1} \land x + y + z \leq 5 \land (y - z \leq 3)\mathtt{:2}$ can be fully interpreted in $\mathsf{IPC}(\mathcal{B} \times \mathsf{O})$.
To understand how $\mathsf{IPC}$ proceeds, we must first introduce two new concepts: the projection function and propagators.

$\mathsf{IPC}(A)$ expects $A$ to provide an additional projection function of the variables onto intervals, defined as $\mathit{project}: (A \times \vars) \to I$.
The function $\mathit{project}(a,x)$ must over-approximate the set of solutions of $x$ in $a$, \emph{i.e.}, for each value $v$ that takes $x$ in $\gamma_A(a)$, $v \in \gamma_I(\mathit{project}(a,x))$.
The interval lattice might be defined over rational numbers $\mathbb{Q}$, floating point numbers $\mathbb{F}$ or integers $\mathbb{Z}$ depending on $A$.
Projection can be implemented directly in many arithmetic domains such as boxes and octagons, but it is sometimes more difficult as it is the case in polyhedra.
In the case of the direct product, projection is defined as $\mathit{project}((a_1,\ldots,a_n), x) = \mathit{project}_1(a_1, x) \sqcup \ldots \sqcup \mathit{project}_n(a_n, x)$.
If the variable $x$ does not belong to an abstract domain, $\mathit{project}_i(a_i,x)$ maps to $\bot$.
The projection is defined on integer intervals if any of the underlying abstract domains projects $x$ onto integers---since integers are more constrained than other types.
In the cases of $\mathbb{F}$ and $\mathbb{Q}$, rational numbers are preferred as they are more precise.

A propagator on an abstract domain $A$ is an extensive function $p: A \to A$ implementing an inference algorithm for a given constraint.
The $\mathit{closure}$ operator of any abstract domain $A$ can be viewed as a propagator on $A$.
The difference is that a propagator implements a single constraint whereas abstract domains support a larger constraint language.
To illustrate propagators, we consider a propagator for the constraint $x \geq y$ generically on an abstract domain $A$ with a projection function.
\begin{displaymath}
\llbracket x \geq y \rrbracket = p_\geq = \lambda a.a \sqcup_A \llbracket x \geq y_\ell \rrbracket_A \sqcup_A \llbracket y \leq x_u \rrbracket_A
\end{displaymath}
with $\mathit{project}(a, x) = [x_\ell..x_u] \text{ and } \mathit{project}(a, y) = [y_\ell..y_u]$.
For instance, given $\mathit{project}(a, x) = [1..2] \text{ and } \mathit{project}(a, y) = [2..3]$, and the constraint $x \geq y$, we obtain $\mathit{project}(p_\geq(a), x) = [2..2] \text{ and } \mathit{project}(p_\geq(a), y) = [2..2]$.
We notice that this propagation step is extensive since we have $a \leq p_{\geq}(a)$.
The constraint $x + y + z \leq 5$ can be implemented by a similar propagator.
The propagation performed on $x$, $y$ and $z$ will be automatically communicated to the direct product $\mathcal{B} \times \mathsf{O}$, which in turn will communicate the new bounds to the box and octagon components.
However, we must solve a small technical issue.
In the constraint $c_3$, we do not wish to propagate new bounds on $x$ in octagons since there is no octagonal constraint involving $x$.
In a propagator defined similarly to $p_\geq$, the variable $x$ would be added in the octagon.
We overcome this issue with a function $\mathit{embed}: A \times A \to A$ defined as $\mathit{embed}(a_1, a_2) = a_1 \sqcup a_2$ if $\mathit{vars}(a_2) \subseteq \mathit{vars}(a_1)$\footnote{The function $\mathit{vars}: A \to \mathcal{P}(\vars)$ can be generically added to any abstract domain by capturing the variables of a formula $\varphi$ before it is interpreted into an element of $A$.}, and $\mathit{embed}(a_1, a_2) = a_1$ otherwise.
We define this function coordinatewise on the direct product.
The corrected version of the propagator $p_\geq$ is given as follows:
\begin{displaymath}
\llbracket x \geq y \rrbracket = p_\geq = \lambda a.\mathit{embed}_A(a, \llbracket x \geq y_\ell \rrbracket_A) \sqcup_A \mathit{embed}_A(a, \llbracket y \leq x_u \rrbracket_A)
\end{displaymath}

Besides extensiveness, we usually require a propagator to over-approximate the set of solutions (soundness), \emph{i.e.}, it should not remove solutions of the logical constraint, in order to guarantee the correctness of the solving algorithm, formally $\gamma(p(a)) \supseteq \cdi{\varphi}$.
Finally, we associate to each propagator $p$ a $state_p$ function which is defined similarly to the one of abstract domain.
In particular, an element $a$ is a solution of $p$ if $state_p(a) = \ttrue$.

Putting all the pieces together, we obtain the lattice $\mathit{Pr} = \langle \mathcal{P}(\mathit{Prop}), \subseteq \rangle$ where $\mathit{Prop}$ is the set of all propagators (extensive and sound functions).
The \textit{interval propagators completion} of an abstract domain $A$ with projection is given by the Cartesian product $\mathsf{IPC}(A) = \langle A \times \mathit{Pr}, \leq \rangle$ with its operations defined as follows for $(a, P),(a',P') \in \mathsf{IPC}(A)$:
\begin{itemize}
\item $(a,P) \leq (a',P') \Leftrightarrow a \leq_A a' \land P \subseteq P' \qquad\quad (a,P) \sqcup (a',P') \eqdef (a \sqcup_A a', P \cup P')$.
% \item $\pi((a, P)) \eqdef (a)$ and $\rho((a, P), (b))\eqdef (a \sqcup_A b, P)$.
\item $state((a,P)) \eqdef state_A(a) \;\land\;  \bigwedge_{p \in P}~state_p(a)$ which means that we reach a solution when $a$ is a solution for all propagators in $P$.
\item $\gamma((a,P)) \eqdef \bigcup\{ \gamma_A(a') \,|\, a' \geq_A a \land state((a', P))=\ttrue\}$.
\item The function $\llbracket c \rrbracket$ associates the constraint $c$ to its propagator $p_c$ and state function $state_p$.
For example, we can rely on the propagation algorithm \texttt{HC4}~\cite{ilog-revising-1999} which works generically over arbitrary arithmetic constraints.
\item $closure((a, \{p_1,\ldots,p_n\})) \eqdef (\mathbf{fp} (p_1 \circ \ldots \circ p_n)(a), \{p_1,\ldots,p_n\})$.
\item $split((a,P)) \eqdef \{(a', P) \;|\; a' \in split_A(a) \}$.
% \item $split((a,P)) \eqdef \{(a, P)\}$.
\end{itemize}
\noindent
The \textit{propagation step} is realized by computing a fixed point ($\mathbf{fp}$) of $p_1 \ldots p_n$ altogether in $\mathit{closure}$.
We do not require to compute the \textit{least} fixed point as it has no impact on the termination property of the solving algorithm.
There are many possible implementations of $closure$ as shown in~\cite{apt-essence-1999}.
% We propose an event-based propagation algorithm that generalize this propagation step over multiple abstract domains in Section~\ref{event-ad}.
The next lemma explains that $\mathsf{IPC}$ over a direct product of abstract domains results in a sound over-approximation.
% The proof of all lemmas is found in~\ref{appendix-proofs}.

\begin{lemma}
Let $A_1$ and $A_2$ be abstract domains, and $\varphi$ a logic formula.
If $A_1$ and $A_2$ over-approximate the set of solutions $\cdi{\varphi}$, then $\mathsf{IPC}(A_1 \times A_2)$ also over-approximates $\cdi{\varphi}$.
\label{ipc-lemma}
\end{lemma}

\begin{proof}
Let $\bar{i} = 3-i$.
We know that $\mathit{project}_i$ maps to an over-approximated interval of its variables.
This interval view is transferred into $A_{\bar{i}}$ via $\llbracket v \geq l_i \land v \leq u_i \rrbracket_{\bar{i}}$ which over-approximates the constraints as well.
Therefore, only over-approximations are involved during the information exchange and no solution is lost.
\end{proof}

% The abstract domain $\mathsf{L}(\mathsf{IPC}(\mathcal{B}))$ forms the basis of many modern constraint solvers such as \textsf{GeCode}~\cite{gecode} or \textsf{Choco}~\cite{choco}.
% % However, other transformer such as the event-based transformer, are needed to model more precisely these solvers.
% Importantly, the domain transformer $\mathsf{IPC}$ is not specialized to box, but works over any abstract domain providing an interval-based projection function such as octagon or polyhedron.

% We could have a shared product of two octagons $\mathsf{O}_1 \sp \mathsf{O}_2$ if, for instance, two disjoint sets of octagonal variables and constraints could be extracted for better efficiency.
% Hence, our framework also supports offline decomposition techniques such as presented in~\cite{blanchet-static-2003} to efficiently decompose a model into several octagons\footnote{Online decomposition such as in~\cite{singh-making-2015} is also supported if encapsulated inside a single abstract domain.}.

\subsection{Delayed product}
\label{delayed-product}

$\mathsf{IPC}$ is only able to exchange bound constraints although there are often opportunities for stronger cooperation between domains.
We present the \textit{delayed product}, a product inspired by delayed goals in logic programming, that dynamically exchanges specialized constraints between two domains.
We consider again the constraint $x + y + z \leq 5$ in the formula $c_3$.
Whenever the variable $x$ becomes instantiated, meaning that $x = v$ for a value $v$, we can rewrite the constraint to $y + z \leq 5 - v$ and interpret it in octagons for additional propagation.

Let $A_1$ and $A_2$ be abstract domains such that $A_1$ is strictly more expressive\footnote{The constraint language supported by the interpretation function of $A_2$ is included in $A_1$.} than $A_2$, but $A_2$ is supposed to be more efficient on its constraints language.
The delayed product $\mathsf{DP}(A_1, A_2)$ evaluates a set of formulas $F \subseteq \Phi$ into $A_1$ until they become instantiated enough to be supported in $A_2$.
A variable $x$ is instantiated in $a \in A$ whenever $fix(a,x) \eqdef (x_\ell = x_u)$, with $project(a,x) = [x_\ell..x_u]$, holds.
For readability, we write $val(a,x) = v$ with $v$ the value of $x$ in $a$ whenever $fix(a,x)$ holds.
To describe this product, we rely on a rewriting function that replaces every instantiated variable with its value, formally defined as:
\begin{displaymath}
\begin{array}{l}
\varphi \to_a \left\{
\begin{array}{ll}
\varphi[x \to val(a,x)] & \text{ if } \exists{x \in vars(\varphi)},~fix(a,x) \\
\varphi & \text{ otherwise}
\end{array} \right.
\end{array}
\end{displaymath}
\noindent
A formula to be transferred is an element of the lattice $\mathsf{FT}=\pf{\Phi}{\mathit{Bool}}$ where $\mathit{Bool}=\{\ttrue,\ffalse\}$ and $\ffalse \leq \ttrue$.
Let $f \in \mathsf{FT}$, then $f(\varphi)$ is $\ttrue$ if the formula has already been transferred, and $\ffalse$ otherwise.
We write $\mathit{nt}(f) = \{\varphi \;|\; f(\varphi)=\ffalse\}$ the set of non-transferred formulas.
The delayed product $\mathsf{DP}(A_1, A_2) = \langle A_1 \times A_2 \times \mathsf{FT}, \leq \rangle$ is an abstract domain inheriting most operations from the Cartesian product.
The different operations are defined as follows:
\begin{itemize}
% \item $\pi((a_1,a_2,c)) \eqdef (a_1,a_2)$ and $\rho((a_1,a_2,c), (b_1,b_2))\eqdef (a_1 \sqcup_1 b_1, a_2 \sqcup_2 b_2, c)$.
\item $\llbracket \varphi \rrbracket \eqdef \left\{
\begin{array}{ll}
(\bot_1, \llbracket \varphi \rrbracket_2, \{\}) & \text{ if } \llbracket \varphi \rrbracket_2 \text{ is defined } \\
(\llbracket \varphi \rrbracket_1, \bot_2, \{\varphi \mapsto \ffalse \}) & \text{ otherwise}
\end{array} \right.$
% \item $\mathit{state}((a_1,a_2,c)) \eqdef \mathit{state_1}(a_1) \land \mathit{state_2}(a_2)$.
\item $\mathit{closure}((a_1,a_2,c)) \eqdef (a_1,a_2,c) \sqcup \bigsqcup_{\varphi \in \mathit{nt}(C)} \mathit{closure\_one}(a_1, a_2, \varphi)$ \\
$\text{with }\mathit{closure\_one}(a_1, a_2, \varphi) \eqdef \left\{
\begin{array}{l}
(a_1,a_2 \sqcup_2 \llbracket \varphi' \rrbracket_2, \{\varphi \mapsto \ttrue\}) \text{ where } \varphi \to^{*}_{a_1} \varphi' \\
  \qquad \text{ if } \llbracket \varphi' \rrbracket_2 \text{ is defined and } vars(\varphi') \subseteq vars(a_2) \\
(a_1,a_2,\{\varphi \mapsto \ffalse\}) \\
\end{array}\right.$
% \item $\mathit{split}((a_1,a_2,c)) = \{(a_1,a_2,c)\}$
\end{itemize}
\noindent
The condition $vars(\varphi') \subseteq vars(a_2)$ in $\mathit{closure\_one}$ restricts the product to add a constraint only if the variables of the constraint are already defined in the domain.
It enables the user of the domain to decide with better flexibility which variables need to be instantiated before the constraint is transferred.
% \vspace{-0.3cm}
\paragraph{Improved closure}
By over-approximating a constraint $c$, it is possible to interpret it in $A_2$ even before it becomes instantiated enough.
For instance, the constraint $x + y + z \leq 5$ can be over-approximated to $y + z \leq 5 - x_\ell$ with $\mathit{project}(a_1, x) = [x_\ell..x_u]$ since the minimal value that $x$ can ever take is its lower bound.
Let $x$ be a variable in $a_1 \in A_1$, $e$ an arithmetic expression, and $project(a_1,x) = [x_\ell..x_u]$.
We rely on the following rewriting function $\twoheadrightarrow$:
\begin{displaymath}
x \leq e \twoheadrightarrow_{a_1} x_\ell \leq e \qquad\qquad
x \geq e \twoheadrightarrow_{a_1} x_u \geq e
\end{displaymath}

\begin{lemma}
The function $\twoheadrightarrow$ over-approximates the constraints $x \leq e$ and $x \geq e$.
\label{dp-lemma}
\end{lemma}
\begin{proof}
For any value $v$ of $x$, if $v \leq e$ is entailed, then $l \leq e$ is also entailed since $l \leq v \leq e$ (similarly for $x \geq e$).
\end{proof}
We extend the definition of closure to take into account these over-approximations:

\begin{displaymath}
\mathit{closure\_one}(a_1, a_2, \varphi) \eqdef \left\{
\begin{array}{l}
(a_1,a_2 \sqcup_2 \llbracket \varphi' \rrbracket_2, \{\varphi \mapsto \ttrue\}) \text{ where } \varphi \leadsto_{a_1} \varphi' \\
\qquad \text{ if } \llbracket \varphi' \rrbracket_2 \text{ is defined and } vars(\varphi') \subseteq vars(a_2) \\
(a_1,a_2 \sqcup_2 \llbracket \varphi' \rrbracket_2, \{\varphi \mapsto \ffalse\}) \text{ where } \varphi \twoheadrightarrow_{a_1} \varphi' \\
\qquad \text{ if } \llbracket \varphi' \rrbracket_2 \text{ is defined and } vars(\varphi') \subseteq vars(a_2) \\
(a_1,a_2,\{\varphi \mapsto \ffalse\}) \\
\end{array}\right.
\end{displaymath}
In the case of a partial transfer, the formula $\varphi$ is not set to $\ttrue$ since it is not yet fully taken into account into $A_2$.

\subsection{Combining domain transformers}
% \section{A General Cooperation Scheme}
\label{coop-scheme}

In order to complete our cooperation framework, we tackle the case where two domain transformers share abstract domains.
For instance, consider the formula $c_4 \eqdef (x = 0 \lor x = 1) \land x * y \leq 5$.
We can interpret $x = 0 \lor x = 1$ in $\mathsf{L}(\mathcal{B})$, which supports bound constraints with disjunctions, and $x * y \leq 5$ in $\mathsf{IPC}(\mathcal{B})$.
If we combine these two domains in a direct product $\mathsf{L}(\mathcal{B}) \times \mathsf{IPC}(\mathcal{B})$, the underlying box domain of each transformer will not be shared, because the direct product does not exchange information among its components.
Conversely, sometimes it is important for efficiency to keep two abstract elements of the same type separated.
Consider the example of two octagons with each $n$ distinct variables, the closure operator has a complexity of $\mathcal{O}(n^3) + \mathcal{O}(n^3)$ when two octagon elements are created, but $\mathcal{O}((n+n)^3)$ when merged.
Therefore, both possibilities of either merging or keeping the domains separated must be available.
To this aim, we propose the \textit{shared product} which is a direct product with named components and sharing among components.
To make the notation explicit, we define an element of the shared product as a list of abstract domain declarations.
As an example, the previous domain with ($\mathbf{D1}$) and without ($\mathbf{D2}$) a shared box are written as:
\begin{displaymath}
\begin{array}{c c}
\begin{array}{ll}
\mathbf{D1} =& \mathcal{B}~\mathit{box}; \\
&\mathsf{L}(\mathcal{B})~\mathit{lbox}(\mathit{box}); \\
&\mathsf{IPC}(\mathcal{B})~\mathit{ipc}(\mathit{box});
\end{array} \qquad&\qquad
\begin{array}{ll}
\mathbf{D2} =& \mathsf{L}(\mathcal{B})~\mathit{lbox}(\bot_\mathcal{B}); \\
&\mathsf{IPC}(\mathcal{B})~\mathit{ipc}(\bot_\mathcal{B});
\end{array}
\end{array}
\end{displaymath}
The line $\mathsf{L}(\mathcal{B})~\mathit{lbox}(\mathit{box});$ indicates that the underlying box domain of $\mathsf{L}(\mathcal{B})$ is shared and given by the element $\mathit{box}$.
We say that $\mathit{box}$ is a dependency of $\mathit{lbox}$.
Every element must be declared before being used as dependencies.
% Therefore, the network of abstract domains obtained forms a directed acyclic graph.
When no dependency is expected, the parameter is an unnamed bottom element, \emph{e.g.}, $\bot_\mathcal{B}$.
In that case, the boxes underlying $\mathsf{L}(\mathcal{B})$ and $\mathsf{IPC}(\mathcal{B})$ are not shared.
In order to define the shared product, we rely on two functions to respectively project and join the dependencies:
\begin{displaymath}
\pi: A \to A_1 \times \ldots \times A_n \qquad\qquad
\kappa: A \times A_1 \times \ldots \times A_n \to A
\end{displaymath}
\noindent
In the delayed product, we have $\pi((a_1,a_2,c)) = (a_1,a_2)$ and $\kappa((a_1,a_2,c), d_1, d_2) = (a_1 \sqcup d_1, a_2 \sqcup d_2,c)$.
We now define the shared product.

\begin{definition}[Shared product]
The shared product $\langle A_1~x_1(d^1_1,\ldots,d^1_m) \texttt{ ; } \ldots \texttt{ ; } A_n~x_n(d^n_1,\ldots,d^n_m), \leq\rangle$ is a direct product $A_1 \times \ldots \times A_n$ in which the $\mathit{closure}$ operator is interleaved with a reduction operator.
Let $a_i$ be an element of the product and $\pi_i(a_i) = (b_j,\ldots,b_k)$ the dependencies of $a_i$, where $b_\ell = \bot$ if $d^i_\ell = \bot$, for all $j \leq \ell \leq k$.
Then each $\rho_i$ is an idempotent and monotone function defined as:
\begin{displaymath}
\rho_i(a_1,\ldots,a_n) = (a_1,\ldots,a_j \sqcup b_j, \ldots, a_k \sqcup b_k, \ldots, \kappa_i(a_i,a_j,\ldots,a_k), \ldots, a_n)
\end{displaymath}
We define $\rho$ as the fixed point of $\rho_1 \circ \ldots \circ \rho_n$.
This reduction operator is applied when computing the closure: $\mathit{closure}((a_1,\ldots,a_n)) \eqdef \rho(\mathit{closure}_1(a_1),\ldots,\mathit{closure}_n(a_n))$.
The interpretation function can be extended to support named constraints:
\begin{displaymath}
\llbracket c\mathtt{:}x \rrbracket = (\bot_1,\ldots,\llbracket c \rrbracket_i,\ldots,\bot_n) \text{ where } x = x_i
\end{displaymath}
which is simpler to read than the index notation of the direct product.
\label{shared-product}
\end{definition}
We illustrate the two roles of $\rho_i$ with an example.
Let the element $a$ be $(\mathit{box, lbox, ipc}) \in \mathbf{D1}$.
% The fixed point computation of the reduction function $\mathbf{fp}(\rho)(a)$ merges all the box elements together.
% Therefore, after the fixed point, we have $\pi(\mathit{lbox}) = \pi(\mathit{ipc}) = \mathit{box}$.
Consider $\rho_2(\mathit{box},\mathit{lbox},\mathit{ipc}) = (\mathit{box} \sqcup \pi(\mathit{lbox}), \kappa(\mathit{lbox}, \mathit{box}), \mathit{ipc})$, which merges $\mathit{lbox}$ with the rest of the product.
First, $\rho_2$ merges the dependency of $\mathit{lbox}$ into $\mathit{box}$ with $\mathit{box} \sqcup \pi(\mathit{lbox})$.
% In general, $\rho_i$ the dependencies $(b_j,\ldots,b_k)$ of $a_i$ into their initial abstract elements $a_j \sqcup b_j, \ldots, a_k \sqcup b_k$.
% Second, it updates the dependencies $(b_j,\ldots,b_k)$ of $a_i$ with the initial abstract elements with $\kappa_i(a_i, (a_j,\ldots,a_k))$.
Second, $\rho_2$ updates the dependency of $\mathit{lbox}$ with $\mathit{box}$ using $\kappa(\mathit{lbox}, \mathit{box})$.
In general, since we compute a fixed point of $\rho$, which is also the least by the Knaster-Tarski fixed point theorem, the abstract domains and domain transformers are totally merged.
% The functions $\rho_i$ are idempotent and monotone.
% Intuitively, the fixed point of $\rho = \rho_1 \circ \ldots \circ \rho_n$ implies that all domain transformers sharing abstract domains have been merged, but no more.
% Formally, it is stated as follows.
% $\rho$ is monotone since the composition of monotone functions is monotone.
% Due to the Knaster-Tarski fixed point theorem, the fixed point of $\rho$ is the smallest fixed point.
% Since we have a fixed point, there is no additional dependency that can be merged, otherwise we could apply some $\rho_i$, and this would not be a fixed point.

% We included additional results on $\rho$ in~\ref{appendix-shared-product}.
In practice, the dependencies are implemented by using pointers.
Therefore, $\pi$ and $\kappa$ are defined implicitly for all abstract domains.
As in the former example, at any time a new information is available in $\mathit{box}$, it is automatically accessible to both $\mathsf{L}(\mathcal{B})$ and $\mathsf{IPC}(\mathcal{B})$ due to the sharing via pointers.

An advantage of this framework is that no effort is required by a domain transformer to be plugged into the shared product.
Moreover, the transformers are fully compositional w.r.t. the shared product, \emph{i.e.}, they can be combined with any other transformers without being modified.
We will illustrate the shared product in a larger example in the next section.

\section{Case study and evaluation}
\label{impl-experiments}

\paragraph{Flexible job shop scheduling}
\label{fjobshop}

Job shop scheduling is a well-known NP-hard combinatorial problem.
We have $n$ jobs and $m$ machines such that a job $1 \leq j \leq n$ is a series of $T_j$ tasks that must be scheduled on distinct machines in turn.
For each job $j$ and task $1 \leq t \leq T_j$, the duration of the task is written as $d_{j,t} \in \mathbb{Z}$, and the machine on which the task $t$ is performed is written as $m_{j,t} \in \{1,\ldots,m\}$.
The variables of the problem are the starting dates $s_{j,t}$ for every task $t$.
For each job, we must ensure that every task is finished before the next one starts (precedence constraints):
\begin{equation}
\forall{1\leq j\leq n},~\forall{1 \leq t \leq T_j - 1},~s_{j,t} + d_{j,t} \leq s_{j,t+1}
\label{prec-constraint}
\end{equation}
\noindent
%In addition, t
Two tasks of two different jobs must not use the same machine at the same time:
\begin{equation}
\begin{array}{l}
\forall{1\leq i < j \leq n},~\forall{1 \leq t \leq T_i},~\forall{1 \leq u \leq T_j}\\
\qquad\qquad m_{i,t} = m_{j,u} \Rightarrow s_{i,t} + d_{i,t} \leq s_{j,u} \lor s_{j,u} + d_{j,u} \leq s_{i,t}
\end{array}
\label{disjunctive}
\end{equation}
\noindent
The disjunctive constraints ensure each pair of tasks $(t,u)$ using the same machine do not overlap.
Usually, the goal is to find a schedule of the tasks %such that it finishes
finishing as early as possible. Therefore, it is an optimization problem that seeks to minimize the makespan.
%, minimizing the makespan:
\begin{equation}
\forall{1\leq j\leq n},~s_{j,T_j} + d_{j,T_j} \leq \mathit{makespan}
\label{makespan}
\end{equation}
%Therefore, it is an optimization problem that seeks to minimize the makespan.

% The open shop scheduling variant removes the order in which the tasks of a job must be processed.
% Hence, the notion of ``task'' is merged with the one of job, and the precedence constraints in Eq.~\eqref{prec-constraint} are removed.

% The flow shop scheduling variant requires that every task $t_1,\ldots,t_n$ is processed in order on the machine $1,\ldots,m$.
% The order is the same for every job, which is not the case in the job shop scheduling problem.

% Parallel machine scheduling is a problem where each task can be scheduled on several different machines.
% We have the decision variables $p_{j,t,k} \in \{0,1\}$ which is equal to one if the task $t$ is executed on the machine $k$.
% It is subject to the following constraints:
% \begin{equation}
% \forall{1\leq j\leq n},~\forall{1 \leq t \leq T_j},~\sum_{1 \leq k \leq m} p_{j,t,k} = 1
% \label{parallel-machines}
% \end{equation}

The flexible job shop scheduling problem~\cite{brucker-job-shop-1990} generalizes the job shop scheduling to multiple machines.
A task can be scheduled on a possible set of machines which might have different processing times for the same task.
The model is now parametrized by a set of possible machines $M_{j,t} \subseteq \{1,\ldots,m\}$ for each task $t$, and by a duration $dur_{j,t,m} \in \mathbb{Z}$ depending on the task \textit{and} the machine.
The parameter $m_{j,t}$ of the job shop problem becomes a decision variable $m_{j,t} \in M_{j,t}$ modeling on which machine every task $t$ is run.
The duration of a task depends on the machine on which it is run, thus every $d_{j,t}$ becomes a decision variable as well:
%This is captured in the following constraint:
\begin{equation}
\forall{1\leq j\leq n},~\forall{1 \leq t \leq T_j},~\bigvee_{k \in M_{j,t}} m_{j,t} = k \land d_{j,t} = dur_{j,t,k}
\label{alternative}
\end{equation}
Constraints~\eqref{prec-constraint},~\eqref{disjunctive} and~\eqref{makespan} stay syntactically the same but over decision variables instead of parameters.
% The constraint~\eqref{parallel-machines} is implicitly encoded in the domain of the decision variable $m_{j,t} \in M_{j,t}$.

\paragraph{Crafting abstract domains for the flexible job shop}

The abstract domain $\mathsf{L}(\mathsf{IPC}(\mathcal{B}))$ is expressive enough to treat the full flexible job shop scheduling problem.
However, as expected it is not very efficient.
Octagons are more efficient than boxes on precedence constraints.
To achieve that, we build an abstract domain, that we name $\mathbf{FJS_1}$, based on boxes and octagons:
\begin{displaymath}
\begin{array}{l}
\mathcal{B}~\mathit{box}; \\
\mathsf{O}~\mathit{oct}; \\
\mathsf{L}(\mathsf{IPC}(\mathcal{B} \times \mathsf{O}))~\mathit{any}((\mathit{box}, \mathit{oct}));
\end{array}
\end{displaymath}
We note the usage of nested parenthesis $((\mathit{box}, \mathit{oct}))$ in order to define the dependencies of nested abstract domains.
To ease the distribution of constraints in abstract domains, we declare $\mathit{box}$ and $\mathit{oct}$ although they are not shared.
In the case of $\mathit{oct}$, there is another subtlety: it is necessary to declare $\mathit{oct}$ otherwise its $\mathit{closure}$ operator will not be called since $\mathsf{IPC}$ does not call the closure of its underlying domain.
The next step is to distribute each constraint in the components of $\mathbf{FJS_1}$.
% In practice, the creation of an abstract domain and the distribution of constraints is a round-trip process.
At the first sight, octagons are of limited interest because all precedence constraints are defined on three variables.
Nevertheless, for most instances of the flexible job shop, we observe that some tasks can only be executed on one machine, or some tasks take the same time on all machines.
Hence, some precedence constraints are immediately octagonal since the duration is fixed.
Constraints in Eq.~\eqref{prec-constraint} are distributed in $\mathbf{FJS_1}$ as follows:
\begin{displaymath}
\forall{1\leq j\leq n},~\forall{1 \leq t \leq T_j - 1},
\left\{ \begin{array}{ll}
     (s_{j,t} + d' \leq s_{j,t+1})\mathtt{:}\mathit{oct} & \text{ if } \{d'\} = \{\mathit{dur}_{j,t,k} \;|\; k \in M_{j,t}\} \\
     (s_{j,t} + d_{j,t} \leq s_{j,t+1})\mathtt{:}\mathit{any} & \text{ otherwise } \\
    \end{array} \right.
\end{displaymath}
\noindent
It is the same for Eq.~\eqref{makespan}.
All the others constraints can be interpreted in $\mathit{any}$.
In addition, since $\mathsf{IPC}$ relies on the underlying domain to represent the variable's domains, we must add all the variables in the box domain first, for each job $j$ and task $t$:
\begin{displaymath}
(s_{j,t} \leq h \land \mathit{makespan} \leq h \land m_{j,t} \leq \mathit{max}(M_{j,t}) \land d_{j,t} \leq \mathit{max}(\{\mathit{dur}_{j,t,k} \;|\; k \in M_{j,t}\}))\mathtt{:}\mathit{box}
\end{displaymath}
The constant $h$ represents the horizon, which is the latest date at which a task can start.

In $\mathbf{FJS_1}$, we statically dispatch the precedence constraints when creating the model.
Because of the delayed product, we can dynamically dispatch the precedence constraints when the durations become fixed, that is, during the solving process.
Precedence constraints can be solved efficiently in the domain $\mathbf{PREC} = \mathsf{DP}(\mathsf{IPC}(\mathcal{B} \times \mathsf{O}), \mathsf{O})$.
The precedence constraints with three variables are interpreted in $\mathsf{IPC}(\mathcal{B} \times \mathsf{O})$, similarly to $\mathbf{FJS_1}$.
In addition, exact and over-approximations of precedence constraints with two variables are dynamically sent in the octagon element thanks to the delayed product.
To experiment with this idea, we craft the abstract domain $\mathbf{FJS_2}$ as follows:
\begin{displaymath}
\begin{array}{l r}
\mathcal{B}~\mathit{box}; &\\
\mathsf{O}~\mathit{oct}; &\\
\mathbf{PREC}~\mathit{prec}(((\mathit{box}, \mathit{oct})), \mathit{oct}); & \text{precedence constraints Eq.~\eqref{prec-constraint} and \eqref{makespan}} \\
\mathsf{L}(\mathcal{B} \times \mathbf{PREC})~\mathit{no\_overlap}(\mathit{box}, \mathit{prec});& \text{non-overlap constraints Eq.~\eqref{disjunctive}}\\
\mathsf{L}(\mathcal{B})~\mathit{alternatives}(\mathit{box});& \text{machine alternatives constraints Eq.~\eqref{alternative}}
\end{array}
\end{displaymath}
\noindent
The constraints can be annotated with the name of the relevant abstract domains, similarly to what we did for $\mathbf{FJS_1}$.
The formula in Eq.~\eqref{disjunctive} is constituted of an implication and disjunctions that can be interpreted in the abstract domain $\mathit{no\_overlap}$.
The atoms of the formula are either equality constraints ($m_{i,t} = m_{j,u}$) that can be interpreted by the box element $\mathcal{B}$, or precedence constraints that can be interpreted in $\mathbf{PREC}$.
Finally, Eq.~\eqref{alternative} could be interpreted in $\mathit{no\_overlap}$, but since $\mathbf{PREC}$ is not useful for this formula, we can avoid the unnecessary indirection by interpreting this formula in the dedicated $\mathit{alternatives}$ domain.

\paragraph{Implementation and evaluation}

We have implemented the abstract domains and transformers presented above in the constraint solver \textsf{AbSolute}~\cite{pelleau-constraint-2013}, which is programmed in \textsf{OCaml} and available online\footnote{The version of \textsf{AbSolute} used in this paper is accessible at \texttt{\href{https://github.com/ptal/AbSolute/tree/iclp2020}{github.com/ptal/AbSolute/tree/iclp2020}}.}.
Our experiments are replicable and all the results are also publicly available.
One of our design goals was to keep the solver as close as possible to its underlying theory.
To achieve this goal, we relied on \textsf{OCaml} functors, such that each domain transformers is a functor parametrized by its sub-domains.
See~\ref{fjs1-absolute} for an example of the \textsf{OCaml} code modeling the abstract domain $\mathbf{FJS_1}$.

% We evaluated the abstract domains $\mathsf{L}(\mathsf{IPC}(\mathcal{B}))$ and $\mathsf{L}(\mathsf{O})$ on the job shop scheduling problem.
% As expected, $\mathsf{L}(\mathsf{O})$ finds better lower bound than $\mathsf{L}(\mathsf{IPC}(\mathcal{B}))$, but constraint solvers such as \textsf{GeCode} and \textsf{Chuffed} are largely better.
% As the main topic of this paper is cooperation, we focus on the evaluation of the abstract domains $\mathbf{FJS_2}$ and $\mathbf{FJS_3}$ on the flexible job shop scheduling problem.

% \paragraph{Experimental setup}

The experiments are all performed on an Intel(R) Xeon(TM) E5-2630 V4 running at 2.20GHz on GNU Linux.
We evaluate three solvers: \textsf{AbSolute v0.10}, \textsf{GeCode v6.1} \cite{gecode} which is a state of the art propagation-based constraint solver, and \textsf{Chuffed v0.10.4}~\cite{Ohrimenko:2009:PVL:1553323.1553342} which is a hybrid solver between constraint propagation and SAT solving.
\textsf{Chuffed} shows excellent results on scheduling problems including the flexible job shop~\cite{schutt-scheduling-2013}.
% \begin{itemize}
%   \item \textsf{AbSolute v0.10} compiled with the version \texttt{4.09.1+flambda} of the \textsf{OCaml} compiler.
%   \item \textsf{GeCode v6.1}~\cite{gecode} is a state of the art propagation-based constraint solver in C++.
%   \item \textsf{Chuffed v0.10.4}~\cite{Ohrimenko:2009:PVL:1553323.1553342} is a hybrid solver between constraint propagation and SAT solving.
%   It shows excellent results on scheduling problems including the flexible job shop~\cite{schutt-scheduling-2013}.
% \end{itemize}
Since we primarily focus on evaluating the propagation process, we selected a search strategy available in all solvers.
This strategy, that we call \texttt{dms}, assigns the domain of each variable to its lower bound and selects the variables with the smallest domain first (\textit{first-fail strategy}).
Furthermore, we first assign all durations, then all machines, and finally the starting dates variables.
% We also tested the domain/weighted-degree strategy in \textsf{GeCode} and \textsf{Chuffed} as it is known to perform well on scheduling problems~\cite{grimes-solving-2015}.
% In order to increase the strength of the different solvers, we relied on a constraint model of the flexible job shop containing a \textit{cumulative} global constraint in \textsf{GeCode} and \textsf{Chuffed}.
% The strength of \textsf{AbSolute} relies on using cooperation schemes between the box and octagon abstract domains.
% An advantage of \textsf{AbSolute} is that the model stays identical across the different tested abstract domains.
We experimented on two sets of instances, named \texttt{edata} and \texttt{rdata}, due to~\citet{hurink-flexible-jobshop-1994}, which are still challenging today~\cite{schutt-scheduling-2013}.
The difference among the sets is the average ratio of machines available per task, \texttt{edata} has few machines per task, and for \texttt{rdata} most tasks can be scheduled on several machines.
% In increasing order we have the sets \texttt{edata}, \texttt{rdata} and \texttt{vdata}.
% The more machines are available, the harder it is to solve these instances, thus \texttt{vdata} is the hardest set.
Each solver is run once on each instance for a maximum of 10 minutes.

% \paragraph{Analysis of the results}

The results are exposed in Table~\ref{experiments-fjs}.
For each solver, we read in column $\Delta_{LB}$ the percentage of how far is the obtained solution from the best known lower bound.
For example, the value $\mathbf{66}$ in bold in Table~\ref{experiments-fjs} indicates that \textsf{Chuffed} found $66$ strictly better bounds than \textsf{GeCode}.

Firstly, although \textsf{AbSolute} is only a prototype, we observe that on \texttt{edata} it finds 36 bounds that are better than the ones found by \textsf{GeCode}, and 23 bounds better than \textsf{Chuffed}.
This demonstrates that communication between domains brings a computational advantage.
For data sets with more machines, the efficiency of \textsf{AbSolute} drops behind the other solvers.
This is because we do not treat machines in a special way in contrast to \textsf{GeCode} or \textsf{Chuffed} that use a \textit{cumulative} global constraint.

Secondly, the difference between $\mathsf{FJS_1}$ and $\mathsf{FJS_2}$ is less obvious since $\mathsf{FJS_2}$ is only able to find a few better bounds.
%Actually, t
This is explained by \texttt{dms} which fixes the durations at the top of the search tree, thus all over-approximations are exchanged early in the search, and do not impact the propagation of most of the nodes.
However, for the flexible job shop, \texttt{dms} was the best strategy we tried in \textsf{AbSolute}.
Nevertheless, we found that $\mathsf{FJS_2}$ was able to find its best bound 20\% quicker than $\mathsf{FJS_1}$ w.r.t. the number of nodes for about 90\% of the instances.
This confirms that better cooperation leads to better pruning in general.

We believe that this framework achieves modularity because new abstract domains can be seamlessly combined with existing ones in order to treat new constraints.
Besides, the presented abstract domains and transformers \textit{are not specifically designed} for the jobshop scheduling problem.
These transformers are applicable to numerous other problems.
\citet{ziat-combination-2019} combine boxes and polyhedra to solve continuous constraint problems; their product is a particular instance of our delayed product.
Furthermore, the delayed product could also be applied to car sequencing problems which involve linear constraints and octagonal constraints~\cite{brand-sequence-decomposition-2007}.

\begin{table}
\centering
\footnotesize
\begin{tabular}{l | c c c c c | c c c c c}
solver & $\Delta_{LB}$(\%) & $\mathbf{FJS_1}$ & $\mathbf{FJS_2}$ & \textsf{GeCode} & \textsf{Chuffed} & $\Delta_{LB}$(\%) & $\mathbf{FJS_1}$ & $\mathbf{FJS_2}$ & \textsf{GeCode} & \textsf{Chuffed} \\
\hline
 & \multicolumn{5}{c}{\texttt{edata}} & \multicolumn{5}{c}{\texttt{rdata}} \\
\hline
$\mathbf{FJS_1}$ & 20.4 & $\square$ & 0 & 36 & 23 & 46.4 & $\square$ & 0 & 4 & 0 \\
$\mathbf{FJS_2}$ & 20.4 & 1 & $\square$ & 36 & 23 & 46.4 & 2 & $\square$ & 4 & 0 \\
\textsf{GeCode} & 20.9 & 30 & 30 & $\square$ & 0 & 31.7 & 61 & 61 & $\square$ & 0 \\
\textsf{Chuffed} & 12.2 & 43 & 43 & \textbf{66} & $\square$ & 24.2 & 66 & 66 & 66 & $\square$ \\
% \hline
% \multicolumn{9}{c}{\texttt{rdata}} \\
% \hline
% $\mathbf{FJS_2}$ & & & $\square$ & & & & & \\
% $\mathbf{FJS_3}$ & & & & $\square$ & & & & \\
% \textsf{GeCode} (\texttt{dms}) & & & & & $\square$ & & & \\
% \textsf{GeCode} (\texttt{dwd}) & & & & & & $\square$ & & \\
% \textsf{Chuffed} (\texttt{dms}) & & & & & & & $\square$ & \\
% \textsf{Chuffed} (\texttt{dwd}) & & & & & & & & $\square$ \\
% \hline
% \multicolumn{9}{c}{\texttt{vdata}} \\
% \hline
% $\mathbf{FJS_2}$ & & & $\square$ & & & & & \\
% $\mathbf{FJS_3}$ & & & & $\square$ & & & & \\
% \textsf{GeCode} (\texttt{dms}) & & & & & $\square$ & & & \\
% \textsf{GeCode} (\texttt{dwd}) & & & & & & $\square$ & & \\
% \textsf{Chuffed} (\texttt{dms}) & & & & & & & $\square$ & \\
% \textsf{Chuffed} (\texttt{dwd}) & & & & & & & & $\square$ \\
\end{tabular}
\caption{Experiments on the flexible job shop scheduling problem (2 * 66 instances).}
\label{experiments-fjs}
\end{table}

\section{Conclusion and future work}

Abstract constraint solving is an exciting new area of research where the foundation of constraint solving is reformulated as abstract interpretation.
We contribute to this area by developing a modular abstract framework allowing solvers and cooperation schemes to be combined seamlessly.
To this end, we have introduced the \textit{interval propagators completion} and the \textit{delayed product} domain transformers implementing two cooperation schemes.
% These transformers implement different cooperation schemes between domains in a modular way.
Moreover, we have introduced the \textit{shared product} to modularly combine domain transformers.

There are three important perspectives of this work.
The first one is to catch up with ACDCL by incorporating conflict learning in \textsf{AbSolute}, which is crucial for efficiency as notably demonstrated  by lazy clause generation in \textsf{Chuffed}~\cite{Ohrimenko:2009:PVL:1553323.1553342}.
Secondly, an inference mechanism to automatically build the right abstract domain to solve a logical formula would be interesting.
This is not trivial as a formula might be interpretable in several abstract domains, thus expressiveness and efficiency must be taken into account in the inference process.
Finally, it is most often necessary to program a customized search strategy in order to achieve better solving efficiency.
This framework only supports combination of search strategies in a restricted way.
We suggest to rely on \textit{spacetime programming}, a synchronous and concurrent search strategy language operating over lattice structures to integrate search in this framework~\cite{PPDP19-Talbot}.

% \ack We would like to thank the referees for their comments, which helped improve this paper considerably

\label{sect:bib}

{
\small
\bibliographystyle{acmtrans}
\bibliography{reference}}

\newpage
\appendix

\section{$\mathbf{FJS_1}$ in \textsf{AbSolute}}
\label{fjs1-absolute}

We give an example of how to turn $\mathbf{FJS_1}$ into an abstract domain at the implementation-level.
We first create the leaves of the combination, in this case the box and octagon abstract domains:

\begin{verbatim}
  module Box = Box_base(Box_split.First_fail_LB)(Bound_int)
  module Octagon = Octagon.Make(ClosureHoistZ)(Octagon_split.MSLF)
\end{verbatim}
\noindent
These two domains are parametrized by a \textit{split} operator, we further indicate that we need a box over integers, and an octagon over integers as well---\texttt{ClosureHoistZ} is a possible implementation of the closure operator for octagon.
We now encapsulate these domains in the interval propagator completion:

\begin{verbatim}
  module BoxOct = Direct_product(Prod_cons(Box)(Prod_atom(Octagon)))
  module IPC = Propagator_completion(Box.Vardom)(BoxOct)
\end{verbatim}
\noindent
The completion is additionally parametrized by a variable domain (here the same as box) which indicates the domain in which the propagation takes place.
For instance, if we have a completion over integers and floating point numbers (\emph{i.e.}, $\mathsf{IPC}(\mathcal{B}(\mathbb{Z}) \times \mathcal{B}(\mathbb{F}))$), the constraints could be evaluated in a rational domain since it subsumes both integers and floating point numbers.
The completion takes care of the required conversions.

The remaining step is to derive the logic completion of $\textsf{IPC}$, and to gather all components in the shared product:

\begin{verbatim}
  module LC = Logic_completion(IPC)
  module FJS = Shared_product(
    Prod_cons(BoxOct)(
    Prod_cons(IPC)(
    Prod_atom(LC))))
\end{verbatim}
\noindent
This product can then be instantiated with empty abstract domains, and solved with a fixed point algorithm as presented in Section~\ref{abstract-dom}.

This demonstrates that abstract domains are composed in a modular way at the theoretical level, but also at the implementation level.

\end{document}